\documentclass[english]{pfia}

\usepackage{xcolor}
\definecolor{afiablue}{RGB}{61,159,207}
\definecolor{afiared}{RGB}{167,75,68}
\definecolor{afialightblue}{RGB}{158,193,232}

\newcommand{\field}[1]{\textcolor{afiablue}{#1}}
\newcommand{\val}[1]{\textsf{\textcolor{afialightblue}{#1}}}

\usepackage{amsmath}

\title{Fairness of Classifiers in the Presence of Constraints between Features\footnote{preprint version
of paper accepted at \emph{CP 2026}}}

\author{Martin C. Cooper\fup{1}, Imane Bousdira\fup{2}
\\[6pt]
\fup{1} IRIT, University of Toulouse \\
\fup{2} IRIT, Toulouse INP} 

\date{Firstname.LastName@irit.fr}

\usepackage{graphicx} % Required for inserting images
\usepackage{epic}
\usepackage{color}

\usepackage{amssymb}
\usepackage{amsthm}
\usepackage{caption}
\usepackage{fontenc}
\usepackage{epic}
\usepackage{float}
\usepackage{verbatim}

\newtheorem{prop}{Proposition}
\newtheorem{proposition}[prop]{Proposition}
\newtheorem{defn}{Definition}
\newtheorem{definition}[defn]{Definition}
\newtheorem{cor}{Corollary}
\newtheorem{corollary}[cor]{Corollary}
\newtheorem{exmp}{Example}
\newtheorem{example}[exmp]{Example}

\newtheorem{thm}{Theorem}
\newtheorem{theorem}[thm]{Theorem}

\newcommand{\mbb}[1]{\ensuremath\mathbb{#1}}

\begin{document}

\maketitle

\begin{abstract}
In Machine Learning, an accepted definition of fairness 
of a decision taken by a classifier 
is that it should not depend on protected features, such as gender. 
Unfortunately, when constraints exist between features,
such dependencies can be obscured by the constraints.

To avoid this problem, we propose that a decision 
be considered fair if it has a fair explanation.
We define a fair explanation as a  
prime-implicant reason for the 
decision that does not contain any protected feature
(where the constraints are taken into account in
the definition of prime-implicant). 
Surprisingly, ignoring constraints can completely change the fairness of
a decision (according to this definition) even in the absence
of constraints between protected and unprotected features.

Three possible definitions of fairness of a classifier
are that for all its decisions 
(1) there are  only fair explanations,
(2) there is at least one fair explanation, or
(3) changing protected features does not
change the outcome.
We identify the relationships between these different
definitions of fairness and study the computational complexity of
testing fairness of classifiers. 
\end{abstract}


\section{Introduction}

Fairness has gained heightened attention in recent years, driven by the widespread use of artificial intelligence across diverse domains. Concerns about fairness have become more prominent as AI systems increasingly take on roles in high-stakes areas such as hiring decisions \cite{SchumannFMD20}. Moreover, the principle of fairness has been widely emphasized in the field of trustworthy AI \cite{10.1145/3491209}, for example, in the ethics guidelines for trustworthy AI established by the European Union \cite{EUEthicsfortrustworthyAI}. 

Given the importance of fairness, it has been addressed in several studies and numerous definitions have been proposed \cite{Binns18, SaxenaHDRPL19}. One common definition describes it, in the context of decision-making, as the absence of any favoritism or discrimination toward an individual or a group on the basis of their inherent or acquired characteristics \cite{10.1145/3457607}. More specific definitions appear in the literature. 
Verma and Rubin~\cite{VermaR18} group them into three categories: statistical measures \cite{doi:10.1177/0049124118782533, Chouldechova17, Corbett-DaviesP17, DworkHPRZ12, HardtPNS16, KleinbergMR17}, similarity-based measures \cite{DworkHPRZ12, GalhotraBM17, KusnerLRS17} and causal reasoning \cite{KilbertusRPHJS17, KusnerLRS17, NabiS18} which is based on causal graphs representing relationships between features. The statistical notion of fairness seems to be insufficient, as it depends on the data 
 and largely ignores all features of the classified subject except the sensitive ones, which might hide unfairness \cite{VermaR18}. The second category overcomes such issues by not marginalizing over non-sensitive features.

From a different perspective, fairness can be interpreted along two dimensions: distributive fairness and procedural fairness \cite{Grgic-HlacaZGW18, abs-2404-01877}. The former focuses on the outcomes, whereas the latter concerns the fairness of the decision-making process. A formal definition of fairness is an essential ingredient in the construction of ethical decision-making systems. Therefore, we study this notion from a formal standpoint, and our work aligns primarily with the second dimension. Specifically, the fairness of the decision-making process is addressed by studying the presence of fair explanations. In addition, we discuss the well-known definition of Fairness Through Unawareness that falls within the category of similarity-based measures.
In the last-mentioned definition, two instances are considered similar if they differ only on
protected features, and their outcomes should thus be the same.

Our key contribution lies in considering constraints between features when evaluating fairness.
%Our purely logical approach can be viewed as a foundation on which to base
%the investigation of statistics-based fairness criteria.
Constraints may arise from different origins. They can be due to the encoding process: for example, in the one-hot encoding of a categorical feature as a set of boolean features where, by definition, exactly one of these features should be true. Constraints can also have an origin in the real world modeled by the feature space: for example, someone on maternity leave is necessarily a woman and employed. 

We assume that decisions are taken by a classifier as a function of feature-values and that some features, such as gender or race, are considered protected for ethical or legal reasons. It is reasonable to consider that a classifier is fair if none of its decisions depend on the protected features. However, in the presence of constraints between features, this simple definition needs to be refined. In fact, it is clearly incorrect to ignore constraints when assessing fairness. This is illustrated by the following example that involves logical constraints between protected boolean features in a one-hot encoding.

\begin{example} \label{ex0}
Suppose that the function $\kappa(m,f,g)$ returns 1 if an employee is eligible
for a bonus as a function of the boolean features: $m$ (male), $f$ (female),
$g$ (goals achieved). For simplicity, we assume that all employees are
male or female, but not both. Both $m$ and $f$ are protected features.
There is a unique constraint that $m \equiv \neg f$. Machine learning
may produce the model $\kappa(m,f,g) \equiv (m \land g) \lor (f \land g)$.
Without taking into account the constraint, this model could be
deemed unfair since $\kappa(0,0,1) \neq \kappa(1,0,1)$ (two feature-vectors
which differ only in the feature $m$ have different outcomes). 
However, in the constrained feature-space,
$\kappa(m,f,g) \equiv g$ and hence $\kappa$ is clearly fair.
\end{example}

In the absence of constraints, a proof of the unfairness of a
decision $\kappa(x)=c$ is a witness $y$ such that 
$\kappa(y) \neq c$ and instances $x$ and $y$ differ only on the
protected features. In the presence of constraints, the existence
of such a witness is no longer guaranteed, as illustrated by the
following example.

\begin{example} \label{ex0a}
Suppose that $\kappa(f,\ell)$ returns 1 if an employee is eligible
for a bonus as a function of the features: $f$ (female),
$\ell$ (has taken a 3-month parental leave). Only $f$
is a protected feature. However, we assume that
legislation only allows women to take parental leave of 3 months,
so there is a constraint $\ell \rightarrow f$.
%Suppose that $\kappa \equiv \neg f$, which is clearly unfair
%(to everyone). 
Suppose that $\kappa(x)=0$ where $x=(1,1)$ corresponding to 
a women who has taken 3 months parental leave.
There is no instance $y$ differing with $x$
only on the protected feature $f$ and which satisfies the constraint. 
Hence there is no possible witness to prove unfairness
of this decision.
\end{example}

We show in Section~\ref{sec:3} how the problems raised by 
these two examples can be solved by defining fairness 
of decisions in terms of explanations.

The rest of the paper is organized as follows. In Section~\ref{sec:2}, we introduce general notation and concepts that we use throughout the paper. In Section~\ref{sec:3}, we define the fairness of a particular decision in the presence of constraints
(based on the existence of a fair explanation), followed by a study in Section~\ref{sec:4} on how ignoring constraints affects the notion of fairness. The fairness of classifiers as a whole is presented in Section~\ref{sec:5} where we introduce three possible definitions and Section~\ref{sec:6} analyses the relationships between them. Section~\ref{sec:7} is devoted to complexity results. We address in Section~\ref{sec:causality} the orthogonal question of the choice of which features to protect. A discussion is presented in Section~\ref{sec:disc} and Section~\ref{sec:conc} concludes.

\section{Preliminaries} \label{sec:2}

We suppose that ${\mathcal F}$ is a set of $n$ features and that
${\mathcal K}$ is a set of (at least two) classes.
We assume that the set of features ${\mathcal F}$ is partitioned into
a set ${\mathcal P}$ of protected features (such as gender or race)
and a set ${\mathcal N}$ of unprotected features.
We denote by $\mbb{F}$ the feature space, the cartesian product
of the domains of the $n$ features. 
Not all feature-vectors in $\mbb{F}$ may be possible, due to a set of
constraints denoted by ${\mathcal C}$. Constraints may be a
consequence of the encoding or real-world restrictions.
We denote by $\mbb{F}[{\mathcal C}]$ the set of instances in $\mbb{F}$
which satisfy the constraints. We assume throughout the paper that 
$\kappa : \mbb{F}[{\mathcal C}] \rightarrow {\mathcal K}$ is a classifier.

For notational simplicity we associate each feature with a number from 1 to $n$. 
%If ${\mathcal S} = \{i_1,\ldots,i_s\} \subseteq \{1,\ldots,n\} = {\mathcal F}$ 
%is a subset of features and $x = (x_1,\ldots,x_n) \in \mbb{F}[{\mathcal C}]$
%is an instance, we use the notation $x_{\mathcal S}$ to denote the 
%projection of $x \in \mbb{F}$ onto the subset of features ${\mathcal S}$,
%i.e. $x_{\mathcal S} = (x_{i_1},\ldots,x_{i_s})$.
We can view an instance $x = (x_1,\ldots,x_n)$ %\in \mathbb{F}$ 
as the set of literals
$\{(i,x_i) : i \in {\mathcal F}\}$. If ${\mathcal S} \subseteq {\mathcal F}$,
and $x \in \mathbb{F}$, we use the notation $x_{\mathcal S}$ to denote
the partial assignment $\{(i,x_i) : i \in {\mathcal S}\}$ to features in ${\mathcal S}$.

In the context of eXplainable Artificial Intelligence (XAI),
a formal notion of explanation of a decision has 
emerged~\cite{AmgoudB22,AudemardBBKLM21,BarceloM0S20,BassanAK24,Joao22,OrdyniakPRS24,ShihCD18}.
It is sometimes known as sufficient reason, %or prime-implicant explanation,
but we will use the term abductive explanation (AXp).
\begin{definition}  \label{def:AXp}
A \emph{weak AXp} of a decision 
associated with a classifier-instance pair $\langle \kappa, x \rangle$
is a subset $S$ of the features such that 
$\forall y \in \mbb{F}[{\mathcal C}]$, $y_S=x_S \rightarrow \kappa(y)=\kappa(x)$.
An \emph{AXp} is a subset-minimal weak AXp.
\end{definition}
In the definition of an AXp, the constraints ${\mathcal C}$ are taken into account by only considering 
instances $y$ (agreeing with $x$ on $S$) that satisfy 
the constraints~\cite{AudemardLMS24a,CooperM21,GorjiR22,YuISN023}.

Depending on the context and for simplicity of presentation,
it may be convenient to consider the partial assignment $x_S$, 
rather than the set $S$, as the abductive explanation.
We will also use the notion of the \emph{coverage} of
an explanation $S$, the set of instances in $\mbb{F}[{\mathcal C}]$
that agree with $x$ on $S$.

\section{Fair Decisions in the Presence of Constraints} \label{sec:3}

In this section we consider the fairness of a particular decision
taken by a classifier $\kappa$. Given a decision $\kappa(x)=c$ 
in an \emph{unconstrained} feature space $\mbb{F}$, we
% consider an explanation to be any AXp (abductive explanation). We 
can view an AXp as a prime-implicant explanation of the decision:
it is a sufficient reason for the decision that is not
subsumed by any other sufficient reason.
In an unconstrained feature space, $A$ subsumes $B$ iff $A \subseteq B$. 
In the presence of constraints ${\mathcal C}$, subsumption is more subtle:
when explaining a decision $\kappa(x)=c$,
$A$ \emph{subsumes} $B$ iff the partial assignment $x_B$
implies the partial assignment $x_A$ in $\mathbb{F}[{\mathcal C}]$
(i.e. $\forall y \in \mathbb{F}[{\mathcal C}], (y_B=x_B) \rightarrow (y_A=x_A)$).
$A$ \emph{strictly subsumes} $B$ iff $A$ subsumes $B$ but $B$ does not
subsume $A$. We consider that an explanation should not be strictly subsumed by
another explanation. This is equivalent to saying that its coverage is 
not a proper subset of the coverage of any other explanation.
Since we do not want to include any unnecessary features, we also impose
subset minimality. This leads to the following definition~\cite{CooperA23}.

\begin{definition}  \label{def:PI}
A \emph{PI-explanation}
(prime-implicant explanation) of a decision $\kappa(x)=c$ in
a constrained feature space $\mbb{F}[{\mathcal C}]$ is an AXp 
that is not strictly subsumed by another AXp.
\end{definition} 

For example, consider $\kappa(a,b) \equiv a \land b$.
In $\mbb{F}$, there are two PI-explanations of $\kappa(0,0)=0$, namely
$\{a\}$ and $\{b\}$. Now, suppose that ${\mathcal C}$ consists of
the constraint $\neg a \lor b$. Then, in the
constrained feature space $\mbb{F}[{\mathcal C}]$, $a \land b \equiv a$,
so the decision $\kappa(0,0)=0$ has a single
PI-explanation $\{a\}$ in $\mbb{F}[{\mathcal C}]$:
$b=0$ implies $a=0$ in $\mathbb{F}[{\mathcal C}]$, 
but $a=0$ does not imply $b=0$. Thus
$\{a\}$ strictly subsumes $\{b\}$ since $a=0$ covers (explains)
two decisions $\kappa(0,0)=0$ and $\kappa(0,1)=0$ whereas
$b=0$ only covers the decision $\kappa(0,0)=0$ since $(1,0) \notin \mathbb{F}[{\mathcal C}]$.

\begin{definition}
A PI-explanation is \emph{fair} if it does not contain any
protected feature.
We say that a decision is \emph{existentially fair} if it has a fair PI-explanation.
We say that a decision is \emph{universally fair} if all of its PI-explanations are fair.
A decision is \emph{unfair} if it has no fair PI-explanation.
\end{definition}

If we consider Example~\ref{ex0}, all decisions taken by the 
classifier are universally fair since they all have the same unique PI-explanation
$\{g\}$. If we consider Example~\ref{ex0a}, if $\kappa(f,\ell)
= \neg f$, then all decisions have the same PI-explanation $\{f\}$
and hence the decision $\kappa(1,1)=0$ is deemed unfair even though
there is no witness to this unfairness (i.e. no $y$ satisfying the
constraints, differing with $(1,1)$ only on the protected feature $f$
and such that $\kappa(y)=1$).

\section{How Ignoring Constraints Alters Fairness of Decisions} \label{sec:4}

In this section, we study how ignoring constraints can
seriously alter evaluations of fairness.
We examine the case where there are constraints between protected features ${\mathcal P}$ and unprotected features ${\mathcal N}$, as well as the cases where constraints occur only within ${\mathcal P}$ or only within ${\mathcal N}$.

\subsection{Constraints between protected and unprotected features}

We consider two examples in which constraints are present
between the protected features and the
unprotected features. One shows that ignoring
constraints can make a universally fair decision appear
unfair. The second example shows that ignoring constraints
can make an unfair decision appear universally fair.

\begin{example} \label{ex1}
Consider a function $\kappa(e,m)$ that returns 1 if a person
is eligible for a free training course as a function of two
features: `in employment' ($e$) and `on maternity leave' ($m$).
Suppose that $\kappa(e,m) \equiv (\neg e) \land (\neg m)$.
Without taking into account any constraints between the two
features, the only PI-explanation for the decision $\kappa(0,0)=1$
is $\{e,m\}$. If $m$ is a protected feature, then this PI-explanation
is unfair and hence the
decision appears unfair. However, there is a constraint that a person
can only be on maternity leave if they are in employment, so 
$(e,m)=(0,1)$ is impossible. Thus, in the constrained 
feature-space $\mbb{F}[{\mathcal C}]$, $\kappa \equiv \neg e$ and the only 
PI-explanation for $\kappa(0,0)=1$ is $\{e\}$ which is fair.
\end{example}

\begin{example} \label{ex2}
Suppose that an adoption agency must decide whether a child $C$
is suitable for adoption by two potential parents $P_1,P_2$.
To simplify, we suppose that the adoption agency uses only
three features to take a decision: $s_1$ ($C$ and $P_1$ 
are of the same race), $s_2$ ($C$ and $P_2$ are of the same race)
and $s$ ($P_1$ and $P_2$ are of the same race).
We consider a decision to be unfair if it uses the feature $s$,
since this could be a discrimination against mixed-race
couples, but it is legitimate to favour adoption by parents
of the same race as the child.
There is a logical constraint that if two of $s_1, s_2, s$ are 
true, then the third is also true (since in this case all
three persons are of the same race).
The adoption agency uses the rule
$\kappa(s_1,s_2,s) \equiv (s_1 \land s_2) \lor (\neg s_1 \land \neg s_2 \land s)$.
%The agency could argue that if the child is of a different
%race to the two parents, then it is preferable that the
%parents have the same race.
Consider the decision $\kappa(1,1,1)=1$.
The only PI-explanation for this decision without taking into
account the constraints, is $\{s_1,s_2\}$ which is fair.
However, when we take into account the constraints,
we can easily deduce that $\kappa(s_1,s_2,s) \equiv s$
in $\mbb{F}[{\mathcal C}]$;
hence the only PI-explanation of $\kappa(1,1,1)=1$ when taking
into account the constraints is $\{s\}$ which is unfair.
\end{example}

\subsection{Constraints limited to only protected or only unprotected features}

It is perhaps not surprising that fairness can change
when constraints are ignored between protected and unprotected features.
We now consider the cases in which there are constraints
either only between protected features or only
between unprotected features. The following examples show that
ignoring such constraints can give a false impression of unfairness.

\begin{example} \label{ex3}
Let $\kappa(f,p,g)$ be a function to decide whether to give a bonus
to an employee based on three features: $f$ (the employee is female),
$p$ (the employee is pregnant) and $g$ (the employee has achieved
their goals). The features $f$ and $p$ are both protected.
Suppose that $\kappa(f,p,g) \equiv g \land (f \lor \neg p)$
and that there is a single constraint: only women can be pregnant.
This means that in $\mbb{F}[{\mathcal C}]$, $\kappa(f,p,g) \equiv g$
(since $f \lor \neg p$ is always true). Without taking into 
account the constraint, the only PI-explanation of $\kappa(1,1,1)=1$
is $\{f,g\}$ which is unfair since $f$ is a protected feature. 
However, in $\mbb{F}[{\mathcal C}]$ 
the only PI-explanation of $\kappa(1,1,1)=1$ is $\{g\}$ which is fair.
\end{example}

\begin{example} \label{ex4}
Let $\kappa(f,s,e)$ be a function to decide whether
a person should provide a medical certificate,
as a function of three features: $f$ (the person is female),
$s$ (the person is on sick leave), $e$ (the person is employed).
Suppose that $\kappa(f,s,e) \equiv (f \land s) \lor (s \land e)$
where $f$ is the only protected feature and there is a single constraint:
$\neg s \lor e$ (to be on sick leave, one needs to be employed). 
In $\mbb{F}$, the decision $\kappa(1,1,1)=1$
has two PI-explanations, namely $\{f,s\}$ (which is unfair)
and $\{s,e\}$. However, in $\mbb{F}[{\mathcal C}]$
$\kappa(f,s,e) \equiv s$ and there is a single PI-explanation
of $\kappa(1,1,1)=1$, namely $\{s\}$, which is fair.
\end{example}

These examples show that unfair PI-explanations can disappear
when constraints (between protected features, as in Example~\ref{ex3}, 
or between unprotected features, as in Example~\ref{ex4})
are taken into account.
On the other hand, we now show that fair PI-explanations cannot disappear
when taking into account constraints, provided that
there are no constraints between protected and unprotected features. 

\begin{proposition}  \label{prop:noPN}
Suppose that there are no constraints in ${\mathcal C}$
between protected and unprotected features. For any
$x \in \mbb{F}[{\mathcal C}]$, if the decision $\kappa(x)=c$ 
has a fair PI-explanation in $\mbb{F}$, then it also
has a fair PI-explanation in $\mbb{F}[{\mathcal C}]$.
\end{proposition}

\begin{proof}
If the decision $\kappa(x)=c$ has a fair PI-explanation in $\mbb{F}$,
then the set of unprotected features ${\mathcal N}$ 
is a weak AXp in $\mathbb{F}$.
Let $\mbb{P}$ be the set of all assignments to the features ${\mathcal P}$,
and let $\mbb{P}[{\mathcal C}]$ denote the set of assignments in $\mbb{P}$ 
satisfying the constraints between the protected features.
By definition of a weak AXp in $\mbb{F}$, 
\[
\forall v \in \mbb{P}, \ \kappa(v \cdot x_{\mathcal N}) = c
\]
where $v \cdot x_{\mathcal N}$ is the feature-vector that assigns $v$
to the features of ${\mathcal P}$ and $x_{\mathcal N}$ to
the features of ${\mathcal N}$. Since $\mbb{P}[{\mathcal C}] \subseteq \mbb{P}$,
it immediately follows that
\[
\forall v \in \mbb{P}[{\mathcal C}], \ \kappa(v \cdot x_{\mathcal N}) = c
\]
and hence ${\mathcal N}$ is a weak AXp in $\mbb{F}[{\mathcal C}]$
by Definition~\ref{def:AXp}.
The only way that there can be no
fair PI-explanation in $\mbb{F}[{\mathcal C}]$ is if this weak AXp ${\mathcal N}$
is subsumed by an unfair PI-explanation $T$. 

Suppose, for a contradiction, that $T$ is an unfair PI-explanation that subsumes 
${\mathcal N}$. $T$ is therefore an AXp and $T \cap {\mathcal P} \neq \emptyset$. 
To subsume ${\mathcal N}$, $T$ must cover all assignments of the form
$v \cdot x_{\mathcal N}$ for all $v \in \mbb{P}[{\mathcal C}]$.
This implies that $v_{T'} = x_{T'}$ (where $T' = T \cap {\mathcal P}$) 
for all $v \in \mbb{P}[{\mathcal C}]$.
Since there are no constraints between ${\mathcal P}$ and ${\mathcal N}$,
the coverage of $T$ is then necessarily equal to the coverage of
$T \setminus {\mathcal P}$, contradicting subset-minimality of the AXp $T$.
This contradiction shows that there is necessarily a fair PI-explanation
of $\kappa(x)=c$ in $\mbb{F}[{\mathcal C}]$.
\end{proof}

Example~\ref{ex3} showed that fair PI-explanations can appear
when taking into account constraints, even when these constraints exist
uniquely between protected features.
We now show that the existence of a fair PI-explanation is independent of
constraints between the unprotected features.

\begin{proposition} \label{prop:notP}
Suppose that the only constraints in ${\mathcal C}$ are
between unprotected features.
The existence of a fair PI-explanation of a decision $\kappa(x)=c$,
where $x \in \mathbb{F}[{\mathcal C}]$,
is independent of the constraints ${\mathcal C}$.
\end{proposition}

\begin{proof}
We assume that the only constraints are between unprotected features.
There is a fair PI-explanation iff ${\mathcal N}$ is a
weak AXp since (1) if ${\mathcal N}$ is a weak AXp,
it cannot be subsumed by an AXp $T$ that contains
protected features (by the same argument as in the proof of Proposition~\ref{prop:noPN})
and (2) if some subset of ${\mathcal N}$ is a PI-explanation
then it is an AXp, and so
${\mathcal N}$ is necessarily a weak AXp. 

By definition, ${\mathcal N}$ is a weak AXp in $\mbb{F}$ iff 
\begin{equation} \label{eq1}
\forall v \in \mbb{P}, \ \kappa(v \cdot x_{\mathcal N}) = c
\end{equation}
Since there are no constraints on ${\mathcal P}$, 
$\mbb{P}[{\mathcal C}] = \mbb{P}$, where $\mbb{P}[{\mathcal C}]$ and $\mbb{P}$
are as defined in the proof of Proposition~\ref{prop:noPN}. It follows
that ${\mathcal N}$ is a weak AXp in $\mbb{F}[{\mathcal C}]$ 
iff this same condition (\ref{eq1}) is satisfied. 
Hence, the existence of
a fair PI-explanation is independent of the constraints ${\mathcal C}$.
\end{proof}

Propositions~\ref{prop:noPN} and \ref{prop:notP} concern the existence
of fair PI-explanations. We now turn our attention
to the existence of unfair PI-explanations. 
The following example shows that taking into
account constraints between unprotected features
can make an unfair PI-explanation appear.

\begin{example} \label{ex5} 
Consider $\kappa(m,n) \equiv (n=1) \lor (m \land (n=0))$
which is a function of two features: $m$ (the person is male) and
$n$ their number of spouses (where $n \in \mbb{N}$). 
The protected feature is $m$ and there is a constraint
on the unprotected feature $n$, namely $n \leq 1$.
The decision $\kappa(1,1)=1$ has a single PI-explanation
in $\mbb{F}$, namely $\{n\}$ which is fair. However,
$\kappa(m,n) \equiv (n=1) \lor m$ in $\mbb{F}[{\mathcal C}]$, 
and so has two PI-explanations in $\mbb{F}[{\mathcal C}]$, namely
$\{m\}$ (which is unfair) and $\{n\}$.
\end{example}

We have seen in Example~\ref{ex5} that a decision may have an unfair
PI-explanation in $\mbb{F}[{\mathcal C}]$ but not in $\mbb{F}$.
However, when there are constraints only on protected features
(rather than only on unprotected features), this is no longer possible.

\begin{proposition} \label{propunfair}
If the constraints are only on protected features and there
is an unfair PI-explanation in $\mbb{F}[{\mathcal C}]$,
then there is an unfair PI-explanation in $\mbb{F}$.
\end{proposition}

\begin{proof}
Suppose that constraints are only on features in ${\mathcal P}$
and that $T$ is an unfair PI-explanation of the decision
$\kappa(x)=c$. Let $T'=T \cup {\mathcal P}$. $T'$ is
necessarily a weak AXp in $\mbb{F}[{\mathcal C}]$.
Since $T'$ fixes the values of all protected
features to the same values as in $x$, it follows that
$T'$ is also a weak AXp in $\mbb{F}$, since constraints
only concern features in ${\mathcal P}$ whose values are fixed.

For there to be no unfair PI-explanation in $\mbb{F}$, the
weak AXp $T'$ must be strictly subsumed in $\mbb{F}$ by
a fair PI-explanation $S \subseteq {\mathcal N}$. 
Strict subsumption in $\mbb{F}$ is equivalent
to strict inclusion, so we must have $S \subset T'$.
Indeed, since $S \subseteq {\mathcal N}$, 
${\mathcal P} \cap {\mathcal N} = \emptyset$, $T \cap {\mathcal P}  \neq\emptyset$ and
$T'=T \cup {\mathcal P}$, we can deduce that $S \subset T$.
Since $S$ is a weak AXp in $\mbb{F}$, it is also a weak AXp
in $\mbb{F}[{\mathcal C}]$. But then this contradicts
the subset-minimality of PI-explanation $T$.
\end{proof}

\subsection{Summary of the effects of taking into account/ignoring constraints}

\begin{figure}
\thicklines \setlength{\unitlength}{1.3pt}
\newsavebox{\arrowdownP}
\savebox{\arrowdownP}(10,20){
\begin{picture}(10,20)(0,0)
\put(5,0){\line(0,1){20}} \put(5,0){\line(1,1){8}} \put(5,0){\line(-1,1){8}}
\end{picture}
}
\newsavebox{\arrowupP}
\savebox{\arrowupP}(10,20){
\begin{picture}(10,20)(0,0)
\put(5,0){\line(0,1){20}} \put(5,20){\line(1,-1){8}} \put(5,20){\line(-1,-1){8}}
\end{picture}
}
\newsavebox{\arrowdownN}
\savebox{\arrowdownN}(10,20){
\begin{picture}(10,20)(0,0)
\put(5,0){\line(0,1){20}} \put(5,0){\line(1,1){8}} \put(5,0){\line(-1,1){8}}
\put(1,10){\line(2,1){8}}
\end{picture}
}
\newsavebox{\arrowupN}
\savebox{\arrowupN}(10,20){
\begin{picture}(10,20)(0,0)
\put(5,0){\line(0,1){20}} \put(5,20){\line(1,-1){8}} \put(5,20){\line(-1,-1){8}} \put(1,6){\line(2,1){8}}
\end{picture}
}
\begin{picture}(210,145)(35,-5)
\put(50,20){\usebox{\arrowdownN}}
\put(90,20){\usebox{\arrowupP}} 
\put(130,20){\usebox{\arrowdownN}} 
\put(170,20){\usebox{\arrowupN}} 
\put(50,100){\usebox{\arrowdownP}} 
\put(90,100){\usebox{\arrowupN}} 
\put(130,100){\usebox{\arrowdownP}} 
\put(170,100){\usebox{\arrowupP}} 

\put(40,5){\framebox[8.cm]{existence of an unfair PI-explanation in $\mbb{F}[{\mathcal C}]$}} 
\put(40,50){\framebox[8.cm]{existence of an unfair PI-explanation in $\mbb{}{F}$}} 
\put(40,85){\framebox[8.cm]{existence of a fair PI-explanation in $\mbb{F}[{\mathcal C}]$}} 
\put(40,130){\framebox[8.cm]{existence of a fair PI-explanation in $\mbb{F}$}} 
\put(35,110){\makebox(0,0){(a)}} 
\put(35,30){\makebox(0,0){(b)}} 

\put(70,30){\makebox(0,0){${\mathcal P}$}} 
\put(110,30){\makebox(0,0){${\mathcal P}$}} 
\put(150,30){\makebox(0,0){${\mathcal N}$}} 
\put(190,30){\makebox(0,0){${\mathcal N}$}} 
\put(70,110){\makebox(0,0){${\mathcal P}$}} 
\put(110,110){\makebox(0,0){${\mathcal P}$}} 
\put(150,110){\makebox(0,0){${\mathcal N}$}} 
\put(190,110){\makebox(0,0){${\mathcal N}$}} 

\end{picture}
 
\caption{The conservation (or not) of fair/unfair PI-explanations
of $\langle \kappa,x \rangle$ (for $x \in \mathbb{F}[{\mathcal C}]$)
when constraints are taken into account (downward arrows) 
or ignored (upward arrows). The label ${\mathcal P}$
or ${\mathcal N}$ means that constraints are uniquely on 
features in ${\mathcal P}$ or ${\mathcal N}$, respectively.}  \label{fig:trans}
\end{figure}
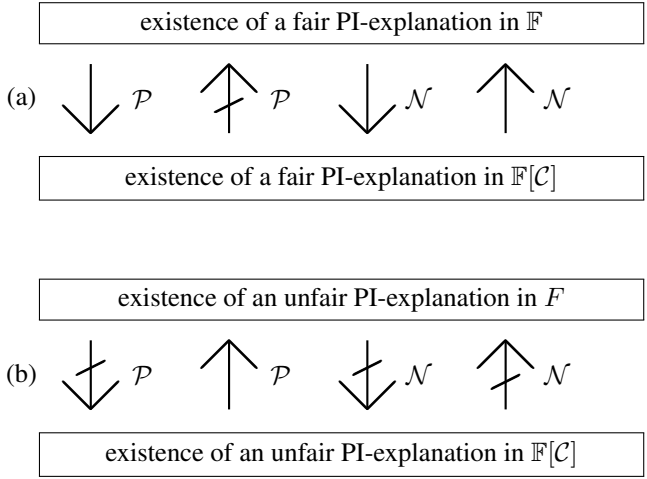

As a conclusion to this section, we can highlight
the main lesson that we have learnt: it is imperative to take
into account constraints when judging whether decisions are fair.
Examples~\ref{ex1} and \ref{ex2} demonstrate that neither the existence of
fair nor the existence of unfair PI-explanations is conserved by taking into account or ignoring
constraints between protected and unprotected features.
We can also point out some positive results that have
some implications on the complexity of detecting fairness. 
Proposition~\ref{prop:notP} tells us that,
when protected features are not constrained, we can test
the existence of a fair PI-explanation of a decision by ignoring
the constraints.
When the constraints concern only the protected features,
there are conditions which are sufficient but not necessary:
(1) the existence of a fair PI-explanation in $\mbb{F}$ implies
the existence of a fair PI-explanation in $\mbb{F}[{\mathcal C}]$,
and (2) the non-existence of an unfair PI-explanation in $\mbb{F}$ implies 
the non-existence of an unfair PI-explanation in $\mbb{F}[{\mathcal C}]$.
Figure~\ref{fig:trans} summarizes the conservation (or not) of
the existence of a fair/unfair PI-explanation when there are no constraints
between protected and unprotected features. Arrows represent implication. The four (non)-implications in Figure~1(a) follow, respectively, from Proposition~\ref{prop:noPN}, Example~\ref{ex3}, Proposition~\ref{prop:noPN}, Proposition~\ref{prop:notP}.
The four (non)-implications in Figure~1(b) follow, respectively, from Example~\ref{ex3}, Proposition~\ref{propunfair}, Example~\ref{ex4}, Example~\ref{ex5}.

\section{Fairness of Classifiers in the Presence of Constraints} \label{sec:5}

We now study the fairness of classifiers as a whole, rather than the fairness of an individual decision. In this regard, we discuss an existing definition and introduce new ones.

\subsection{Constrained FTU as a measure of fairness}

We recall the definition of Fairness Through Unawareness (FTU), which is a classic rule for judging fairness of classifiers in unconstrained feature spaces~\cite{weller-nipsw16,KusnerLRS17,IgnatievCSHM20}:
\[
\forall x,y \in \mbb{F}, \ x_{\mathcal N} = y_{\mathcal N} \ \rightarrow \ 
\kappa(x) = \kappa(y)
\]

When constraints exist, we cannot directly apply FTU as stated above, not least because 
the values of $\kappa(y)$ can be considered as undefined for $y \notin \mbb{F}[{\mathcal C}]$.
As shown in Example~\ref{ex0}, when constraints are overlooked, 
a fair classifier may be wrongly considered unfair. 
Accordingly, the definition of FTU in $\mbb{F}[{\mathcal C}]$ is as follows.
\begin{definition} \label{def:constrainedFTU}
A classifier $\kappa: \mathbb{F}[{\mathcal C}] \rightarrow {\mathcal K}$ satisfies
FTU in $\mathbb{F}[{\mathcal C}]$ \emph{(constrained FTU)} if 
\[
\forall x,y \in \mbb{F}[{\mathcal C}], \ x_{\mathcal N} = y_{\mathcal N} \ \rightarrow \ 
\kappa(x) = \kappa(y)
\]
\end{definition}

We can view $\kappa$ as a partial function whose domain is
$\mbb{F}[{\mathcal C}]$. We say that a total
function $\hat{\kappa}$ is a \emph{completion} of $\kappa$
iff it is an extension of $\kappa$ when $\kappa$ is
viewed as a partial function,
i.e. $\forall x \in \mbb{F}[{\mathcal C}]$,
$\hat{\kappa}(x) = \kappa(x)$. 
Hence, fairness may alternatively be defined as the existence of a fair completion \emph{(FTU by completion)}. It turns out that this is the same as constrained FTU, 
as we will show in Proposition~\ref{prop:equiv}.

\subsection{Measures of fairness based on PI-explanations} \label{subsec:PIXp}

Following the definitions of a fair decision given in Section~\ref{sec:3}, we can define the fairness of a classifier (based on the fairness of all its decisions).
\begin{definition} \label{def:EUfairness}
A classifier $\kappa: \mathbb{F}[{\mathcal C}] \rightarrow {\mathcal K}$ satisfies
\begin{itemize}
\item \emph{existential fairness} if $\forall x \in \mbb{F}[{\mathcal C}]$, 
$\langle \kappa, x \rangle$ has a PI-explanation which is fair, and
\item \emph{universal fairness} if $\forall x \in \mbb{F}[{\mathcal C}]$, all 
PI-explanations of $\langle \kappa, x \rangle$ are fair.
\end{itemize}
\end{definition}

Along with constrained FTU, we now have three definitions of fairness of a classifier.
We compare these definitions and show that they are not always equivalent in Section~\ref{sec:6}.
The computational complexity of verifying these three notions of fairness is discussed in Section~\ref{sec:7}.

\section{Relationships between Definitions of Fairness of Classifiers} \label{sec:6}

In this section, we study the relationship between different
definitions of fairness of classifiers. We begin by giving an alternative
characterisation of constrained FTU (i.e. FTU in $\mathbb{F}[{\mathcal C}]$).

\begin{proposition} \label{prop:equiv}
A classifier $\kappa$ satisfies FTU in $\mbb{F}[{\mathcal C}]$ iff 
there exists a completion $\hat{\kappa}$ of $\kappa$ that satisfies
FTU in $\mbb{F}$.
\end{proposition}

\begin{proof}
The `if' direction is trivial, since $\hat{\kappa}$ and $\kappa$
agree on $\mbb{F}[{\mathcal C}]$. Suppose that $\kappa$
satisfies FTU in $\mbb{F}[{\mathcal C}]$. Let $c_0$ denote
an arbitrary class. For all $x \in \mbb{F}$ 
define $\hat{\kappa}(x)$ as follows:
\begin{enumerate}
    \item if there exists $y \in \mbb{F}[{\mathcal C}]$ such that
    $y_{\mathcal N}=x_{\mathcal N}$, then $\hat{\kappa}(x)=\kappa(y)$,
    \item if not, then $\hat{\kappa}(x)=c_0$.
\end{enumerate}
This definition is consistent since $\kappa$ satisfies FTU
in $\mbb{F}[{\mathcal C}]$. By construction, $\hat{\kappa}$
satisfies FTU in $\mbb{F}$.
\end{proof}

\subsection{Constrained FTU and Existential fairness}

Existential fairness of a classifier $\kappa$ implies that ${\mathcal N}$ is a weak AXp and hence that $\kappa$ satisfies FTU in $\mathbb{F}[{\mathcal C}]$. In contrast, 
the following example shows that applying constrained FTU to
pairs of feature-vectors in $\mbb{F}[{\mathcal C}]$
may miss some cases of unfairness. It shows that
FTU in $\mbb{F}[{\mathcal C}]$ does not imply existential fairness.

\begin{example} \label{ex0bis}
Consider $\kappa(a,b,c)$ which is a function of three
boolean features $a,b,c$, where $a$ is the only protected feature.
Suppose that there is the constraint 
$a \equiv (\neg b \land c) \lor (b \land \neg c)$
and that $\kappa(a,b,c)=a$. $\kappa$ trivially satisfies FTU on
$\mbb{F}[{\mathcal C}]$ since, for each assignment to the
set of unprotected features ${\mathcal N} = \{b,c\}$,
there is only one feature-vector in $\mbb{F}[{\mathcal C}]$.
This is despite the fact that the classifier $\kappa$ is
always equal to the value of the protected feature, so
should be considered unfair. For example, $\kappa(1,1,0)=1$
has only one PI-explanation $\{a\}$, since it strictly subsumes $\{b,c\}$. 
Indeed, any decision $\kappa(x)=c$ ($x \in \mbb{F}[{\mathcal C}]$) 
has a single PI-explanation $\{a\}$.
Hence, $\kappa$ is not existentially fair.
\end{example}

We now give a more realistic example of a classifier satisfying constrained FTU but which is not existentially fair since
there are decisions that have no fair PI-explanation.
\begin{example} \label{exFTUnotfair}
Consider a function $\kappa(f,p,b,a)$ used by a company to decide whether its employees can work from home,
according to the values of the following features:
$f$ (the employee is female), $p$ (the employee works part-time),
$b$ (the employee recently took prenatal parental leave),
$a$ (the employee recently took postnatal parental leave).
The protected feature is $f$.
Suppose that 
\[ \kappa(f,p,b,a) \ \equiv \ (b \lor a) \ \land \ (\neg p \lor \neg b \lor \neg a)
\]
i.e. all employees who have recently taken parental leave, except for part-time employees
that took parental leave both before and after the birth.
Furthermore, company policy is that female employees who give birth must take parental
leave both before and after the birth. On the other hand, male part-time employees cannot take
parental leave both before and after the birth.
So ${\mathcal C} = \{f \land a \rightarrow b$, 
$f \land b \rightarrow a$, $\neg(\neg f \land p \land a \land b)\}$.

It is easy to see that $\kappa$ satisfies FTU in  $\mbb{F}[{\mathcal C}]$ since $\kappa$
is defined without reference to the protected feature $f$.
However, due to the constraints ${\mathcal C}$ given above, 
$\kappa$ can be rewritten as a function of all four features: in $\mbb{F}[{\mathcal C}]$ we have
\[ \kappa(f,p,b,a) \ \equiv \ (b \lor a) \land  (\neg p \lor \neg f) 
\]
It can easily be verified that the decision $\kappa(0,1,1,0)=1$
has a single PI-explanation $\{f,b\}$. Similarly, 
the decision $\kappa(1,1,1,1)=0$ has a single PI-explanation, namely $\{f,p\}$.
These PI-explanations are unfair since they include the protected feature $f$. 
\end{example}

\subsection{Existential fairness and Universal fairness}

It is trivial that a classifier satisfying universal fairness is existentially fair.
To see that existential fairness does not imply universal
fairness, we can consider the following example.

\begin{example} \label{ex0ter}
Consider a feature-space composed of two
boolean features $a,b$, in which $a$ is the protected feature.
Suppose that there is a constraint $a \equiv b$ and that 
$\kappa(a,b) \equiv a$. Indeed, due to the constraint, $\kappa(a,b) \equiv b$.
There are two PI-explanations for $\kappa(1,1)=1$,
namely $\{a\}$ and $\{b\}$, and similarly for $\kappa(0,0)=0$.
Thus, each decision has a fair PI-explanation and an unfair PI-explanation.
Hence $\kappa$ is existentially fair but not universally fair.
\end{example}

To give a less contrived example, we can return to the case of decisions taken 
by an adoption agency (as in Example~\ref{ex2}).
\begin{example} \label{exAdoptAgain}
An adoption agency must decide whether a child $C$
is suitable for adoption by two potential parents $P_1$, $P_2$,
based on only three features: $s_1$ ($C$ and $P_1$ 
are of the same race), $s_2$ ($C$ and $P_2$ are of the same race)
and $s$ ($P_1$ and $P_2$ are of the same race).
The feature $s$ is the only protected feature.
There is a logical constraint that if two of $s_1, s_2, s$ are 
true, then the third is also true.
Suppose that the adoption agency uses the rule
$\kappa'(s_1,s_2,s) \equiv (s_1 \land s_2)$. 
Each decision taken by $\kappa'$ has a fair PI-explanation, but all
decisions with $s=0$ also have the unfair PI-explanation $\{s\}$, 
due to the fact that $\kappa'$ allows no mixed-race couples to adopt. 
Hence $\kappa'$ is existentially fair but not universally fair.

It is worth pointing out that there is a non-trivial universally fair solution to this problem,
given by the function $\kappa''(s_1,s_2,s) \equiv (s_1 \lor s_2)$.
\end{example}

\subsection{Absence of constraints between protected and unprotected features}

Examples \ref{ex0bis}, \ref{exFTUnotfair}, \ref{ex0ter} and \ref{exAdoptAgain}
all involved constraints between protected and unprotected features.
We now show that in the absence of such constraints, the
notions of constrained FTU, existential fairness 
and universal fairness are all equivalent.
The following Theorem applies to several
of the examples given above, namely Examples~\ref{ex0}, 
\ref{ex3}, \ref{ex4} and \ref{ex5}. %and \ref{ex:parent}.

\begin{theorem} \label{thm:nocstrs}
When there are no constraints in ${\mathcal C}$ between protected and unprotected features,
the following notions of fairness of a classifier are equivalent:
\begin{enumerate}
    \item universal fairness
    \item existential fairness
    \item constrained FTU (FTU in $\mbb{F}[{\mathcal C}]$)
\end{enumerate}
\end{theorem}

\begin{proof}
$1 \Rightarrow 2$ is trivial (knowing that each instance 
$x \in \mbb{F}[{\mathcal C}]$ has at least one PI-explanation). \\ 
$2 \Rightarrow 3$ follows from the fact that the existence of a  fair 
PI-explanation $S \subseteq {\mathcal N}$ implies that ${\mathcal N}$
is a weak AXp, which in turn implies FTU in $\mbb{F}[{\mathcal C}]$. \\
$3 \Rightarrow 1$: Suppose $\kappa$ satisfies FTU in $\mbb{F}[{\mathcal C}]$.
Let $x \in \mbb{F}[{\mathcal C}]$ and suppose that $T$ is a PI-explanation
of $\langle \kappa,x \rangle$. Let $S = T \cap {\mathcal N}$.
Consider an arbitrary $y \in \mbb{F}[{\mathcal C}]$ such that $y_S = x_S$.
Let $z$ be such that $z_{{\mathcal N}}=y_{{\mathcal N}}$ and 
$z_{{\mathcal P}}=x_{{\mathcal P}}$.
Observe that $z_T = x_T$. We know that $z \in \mbb{F}[{\mathcal C}]$ since
$y$ satisfies the constraints on features in ${\mathcal N}$ and
$x$ satisfies the constraints on features in ${\mathcal P}$.
We have:
\begin{enumerate}
    \item $\kappa(y)=\kappa(z)$ by FTU in $\mbb{F}[{\mathcal C}]$, and
    \item $\kappa(z)=\kappa(x)$ since $T$ is a PI-explanation of $\langle \kappa,x \rangle$
    and $z_T = x_T$.
\end{enumerate}
Hence, $\kappa(y) = \kappa(x)$. Since $y$ was an arbitrary instance in
$\mbb{F}[{\mathcal C}]$ which agrees with $x$ on $S$, we have just shown that
$S \subseteq T$ is a weak AXp. By subset minimality of the PI-explanation $T$
(which is an AXp by Definition~\ref{def:PI}), we can deduce
that $S=T$. Thus $T \subseteq {\mathcal N}$. We can therefore deduce that all PI-explanations
$T$ are fair.
\end{proof}

%%%%%%%%%%%%%%%%%%%%%%%%%%%%%%%%%%%%%%%%%%%%%%%%%%%%%%%%%%%%%%%%%%%%%%%%%%%%%%%%%%%%%%%%%%
\subsection{Loose constraints}

Anticipating Section~\ref{sec:7}, testing constrained FTU is
computationally simpler than testing existential fairness.
We therefore investigate circumstances under
which these two notions are equivalent.
In the following, we identify a specific condition on the set of constraints which guarantees that FTU and the existence of a fair PI-explanation are equivalent. 

\begin{definition}  \label{def:loose}
Consider a set of features ${\mathcal F}$ partitioned into a set of protected features
${\mathcal P}$ and a set of unprotected features ${\mathcal N}$.  
A set of constraints ${\mathcal C}$ is \emph{loose} (with respect to the partition $({\mathcal P}, {\mathcal N})$ of the set of features) at $x \in \mathbb{F}[{\mathcal C}]$ if
for all  $p \in {\mathcal P}$, $x_{\{p\}}$ does not strictly subsume $x_{\mathcal N}$.
The set of constraints ${\mathcal C}$ is \emph{loose} 
if ${\mathcal C}$ is loose at each $x \in \mathbb{F}[{\mathcal C}]$.
\end{definition}

Given a set of constraints $\mathcal{C}$, the search for unfair assignments that strictly subsumes some fair assignments 
can be carried out off-line as a preliminary step, independently of any classifier, providing insight into the possible entailments caused by imposing constraints within the feature space. 
If no such subsumptions exist (i.e. the constraints are loose), constrained FTU guarantees the existence of a fair PI-explanation.

\begin{proposition} \label{prop:COND1FTU-exist}
If the set of constraints ${\mathcal C}$ is loose,
then FTU in $\mbb{F}[{\mathcal C}]$ implies existential fairness. 
\end{proposition}

\begin{proof}
Suppose for the sake of contradiction that a classifier $\kappa$ satisfies FTU in $\mbb{F}[{\mathcal C}]$ and does not satisfy existential fairness. We have: 
\begin{enumerate}%[(1)]
    \item $\forall x \in \mbb{F}[{\mathcal C}]$, $\langle \kappa, x \rangle$ has a fair AXp (by FTU in $\mbb{F}[{\mathcal C}]$), and
    \item $\exists y \in \mbb{F}[{\mathcal C}]$, $\langle \kappa, y \rangle$ does not have a fair PI-explanation (since $\kappa$ is not existentially fair).
\end{enumerate}
By (1), $\langle \kappa, y \rangle$ has a fair AXp A (i.e. $A \cap {\mathcal P} = \emptyset $). By (2) there must exist an unfair AXp B (i.e. $B \cap {\mathcal P} \neq \emptyset $) such that $y_{B}$ strictly subsumes $y_{A}$ in order for $A$ not to be a PI-explanation. Let $p$ be a protected feature such that $p \in B \cap {\mathcal P}$. It follows that $y_{\{p\}}$ strictly subsumes $y_{A}$, since  
(i) $\forall z \in \mathbb{F}[{\mathcal C}], (z_A=y_A) \rightarrow (z_B=y_B)$, so
$\forall z \in \mathbb{F}[{\mathcal C}], (z_A=y_A) \rightarrow (z_{\{p\}}=y_{\{p\}})$, 
and (ii) $\exists z \in \mathbb{F}[{\mathcal C}], (z_B=y_B) \land (z_A \neq y_A)$, so
$\exists z \in \mathbb{F}[{\mathcal C}], (z_{\{p\}}=y_{\{p\}}) \land (z_A \neq y_A)$.
Since $ A \subseteq \mathcal N$, $y_{\{p\}}$ strictly subsumes $y_{\mathcal N}$ which is impossible due to Definition~\ref{def:loose}.
\end{proof}

To illustrate this, consider Example~\ref{ex0ter}: the set of constraints is loose
(simply because there is no \emph{strict} subsumption between the features), 
the classifier satisfies FTU in $\mbb{F}[{\mathcal C}]$ and so each 
decision has a fair PI-explanation. It is worth pointing out that 
in Example~\ref{ex0ter}, each decision also has an unfair PI-explanation,
which shows that looseness does not imply universal fairness.
Furthermore, under the very reasonable assumption that 
protected features have domains of size greater than one,
the absence of constraints between protected and unprotected features (as in Theorem~\ref{thm:nocstrs}) implies that the set of constraints is loose.

In order to achieve equivalence between the notions of constrained FTU and existential fairness, 
the looseness condition imposed in Proposition~\ref{prop:COND1FTU-exist} may 
not be satisfied in practice for each instance when diverse constraints exist between features. 
Therefore, we propose the following refinement of the previous proposition, whose proof follows a similar argument, in order to enhance applicability.
Note that, by definition~\ref{def:constrainedFTU}, a classifier $\kappa$ satisfies FTU at $x \in \mbb{F}[{\mathcal C}]$ iff $\forall y \in \mbb{F}[{\mathcal C}]$ such that $x_{\mathcal N} = y_{\mathcal N}$, we have $\kappa(x) = \kappa(y)$.

\begin{proposition}  \label{prop:COND2FTU-exist}
If the set of constraints ${\mathcal C}$ is loose at $x \in \mbb{F}[{\mathcal C}]$, then a classifier $\kappa$ satisfies FTU at x implies that the decision $\langle \kappa, x \rangle$ is existentially fair.
\end{proposition}

In other words, the feature space can be partitioned into two sets. For instances where $\mathcal{C}$ is loose, 
constrained FTU implies existential fairness (whereas for instances where $\mathcal{C}$ is not loose, the existence of a fair PI-explanation must be assessed). 
In Example~\ref{exFTUnotfair}, the set of constraints is loose at $x=(0,0,0,0)$. We can conclude, without additional verification, that $\langle \kappa, x \rangle$ has a fair PI-explanation. Specifically, since $\kappa$ satisfies constrained FTU, $\kappa(x)=0$ has a fair AXp $x_{\{a,b\}}$, which is not strictly subsumed by $x_{\{f\}}$. In contrast, the set of constraints is not loose at $y=(1,1,1,1)$ because $y_{\{a,b,p\}}$ is strictly subsumed by 
$y_{\{f\}}$. Hence, we need to check the existence of a fair PI-explanation.  $\kappa(y)=0$ has two AXps: $\{a,b,p\}$ and $\{f,p\}$. As the latter strictly subsumes the former, the only PI-explanation is $\{f,p\}$ which is unfair.

%%%%%%%%%%%%%%%%%%%%%%%%%%%%%%%%%%%%%%%%%%%%%%%%%%%%%%%%%%%%%%%%%%%%%%%%%%%%%%%%%%%%%%%%%%
\subsection{Disentangled classifiers and constraints}
We now provide a condition, involving both 
$\kappa$ and ${\mathcal C}$, that guarantees the existential fairness
of a classifier but is computationally simpler to verify (as will be shown in Section~\ref{sec:7}).
Example~\ref{ex1} illustrates the following definition and proposition.

\begin{definition} \label{def:disentangled}
Given an instance $x \in \mathbb{F}[{\mathcal C}]$,
the decision $\kappa(x)=c$ is \emph{disentangled} (with respect to
constraints ${\mathcal C}$) if the set ${\mathcal N}$ 
of unprotected features is a weak
AXp of decision $\kappa(x)=c$ that is not strictly subsumed by an unfair
weak AXp. A classifier $\kappa$ is \emph{disentangled} in $\mathbb{F}[{\mathcal C}]$
if all of its decisions (for $x \in \mathbb{F}[{\mathcal C}]$) are disentangled.
\end{definition}

\begin{proposition} \label{prop:disentangled}
If a classifier $\kappa$ is disentangled in $\mathbb{F}[{\mathcal C}]$, then $\kappa$ is existentially fair.
\end{proposition}

\begin{proof}
We need to show that a classifier $\kappa$ disentangled in $\mathbb{F}[{\mathcal C}]$
also satisfies existential fairness. Being disentangled implies that for each decision,
${\mathcal N}$ is a weak AXp which is not strictly subsumed by an unfair AXp.
Thus each decision must have an AXp $S \subseteq {\mathcal N}$ which is not
strictly subsumed by an unfair AXp. This $S$ is a fair PI-explanation.
\end{proof}

%%%%%%%%%%%%%%%%%%%%%%%%%%%%%%%%%%%%%%%%%%%%%%%%%%%%%%%%%%%%%%%%%%%%%%%%%%%%%%%%%%%%%%%%%%

\section{Computational Complexity of Testing Fairness} \label{sec:7}

In this section we investigate where testing each different notion 
of fairness lies in the polynomial hierarchy~\cite{AroraBarak}.
One way of defining the polynomial hierarchy is by giving definitions of complete
problems at each level $k$.
Recall that a complete problem for $\Sigma_{k}^{\mathrm{P}}$ is satisfiability for quantified 
boolean formulas with $k-1$ alternations of quantifiers, the first quantifier being $\exists$.
The variant of this problem in which the first quantifier is $\forall$ is
complete for $\Pi_{k}^{\mathrm{P}}$. Since $\Sigma_{1}^{\mathrm{P}}$ is NP and 
 $\Pi_{1}^{\mathrm{P}}$ is coNP, problems in these classes can be solved by SAT solvers.

We first consider the problem of determining the existence of
a fair PI-explanation of a decision. We assume throughout this section that $\kappa$ is a known
function and that the computation of $\kappa(z)$ for $z \in \mbb{F}[{\mathcal C}]$
can be achieved in polynomial time.

The existence of a fair PI-explanation for $\kappa(x)=c$ is equivalent to
\begin{align*}
\exists S \subseteq \mathcal F, \
 \forall y \in \mbb{F}[{\mathcal C}], \ \forall T \subseteq {\mathcal F}, \ 
\exists z \in \mbb{F}[{\mathcal C}] \ ( \ \\
(S \cap {\mathcal P} = \emptyset) \land (y_S=x_S  \rightarrow  \kappa(y)=c) \\  \land \
( \ (T \cap {\mathcal P} = \emptyset) \ \lor \ (z_T = x_T \ 
\land \ \kappa(z) \neq c) \\ \lor \
(z_S=x_S \land z_T \neq x_T) \ \lor \ (y_T=x_T \rightarrow y_S=x_S)
\ ) \ )
\end{align*}
(i.e. there exists a fair weak AXp $S$ which is not strictly subsumed by
an unfair weak AXp $T$). 
Hence, the problem of determining the existence of a fair PI-explanation
of a decision belongs to $\Sigma_3^{\text{P}}$.

Similarly, the existence of an \emph{unfair} PI-explanation for $\kappa(x)=c$ is equivalent to
\begin{align*}
\exists T \subseteq \mathcal F, \
 \forall y \in \mbb{F}[{\mathcal C}], \ \forall S \subseteq {\mathcal F}, \ 
\exists z \in \mbb{F}[{\mathcal C}] \ ( \ \\
(T \cap {\mathcal P} \neq \emptyset) \land (y_T=x_T  \rightarrow  \kappa(y)=c) \\  \land \
( \ (S \cap {\mathcal P} \neq \emptyset) \ \lor \ (z_S = x_S \ 
\land \ \kappa(z) \neq c) \ \lor \
(z_T=x_T \\ \land z_S \neq x_S) \ \lor \ (y_S=x_S \rightarrow y_T=x_T \land  \neg(S \subseteq T)  )
\ ) \ )
\end{align*}
(i.e. there exists an unfair weak AXp $T$ which is not strictly subsumed
by a fair weak AXp $S$). 
It follows that the problem of determining the existence of an unfair PI-explanation
of a decision also belongs to $\Sigma_3^{\text{P}}$.          

Although it is an open question whether the problems of determining the existence
of a fair/unfair PI-explanation are complete for $\Sigma_3^{\text{P}}$,
we can nevertheless state upper bounds on the complexity of testing existential/universal
fairness of decisions.

\begin{proposition}  \label{prop:cpxy1}
The problem of testing whether a decision is existentially fair belongs to
$\Sigma_3^{\text{P}}$. The problem of testing whether a decision is universally
fair belongs to $\Pi_3^{\text{P}}$.
\end{proposition}

\begin{proof}
This follows from the discussion above, together with the fact that testing
universal fairness is the complement of testing the existence of an unfair PI-explanation.
\end{proof}

We now study the complexity of testing fairness of a classifier $\kappa$.

\begin{proposition}
The problem of testing existential fairness of a classifier belongs to
$\Pi_4^{\text{P}}$. The problem of testing universal
fairness of a classifier belongs to $\Pi_3^{\text{P}}$.
\end{proposition}

\begin{proof}
Testing existential fairness involves testing whether 
for all $x \in \mbb{F}[{\mathcal C}]$, the decision $\langle \kappa,x \rangle$
is existentially fair. Since this latter problem belongs to
$\Sigma_3^{\text{P}}$ (by Proposition~\ref{prop:cpxy1}), we can deduce
that testing existential fairness of a classifier belongs to
$\Pi_4^{\text{P}}$. 

On the other hand, since the first quantifier in testing universal
fairness of a decision is already a universal quantifier, adding 
an extra universal quantification over $x \in \mbb{F}[{\mathcal C}]$
does not change the complexity class. Thus, the  problem of testing universal
fairness of a classifier belongs to the same class, namely $\Pi_3^{\text{P}}$
(by Proposition~\ref{prop:cpxy1}), as testing universal fairness of
a single decision.
\end{proof}

On a positive note, we can observe that testing constrained FTU is 
accessible to SAT solvers, as shown by the following proposition.

\begin{proposition} \label{prop:coNPFTU}
Testing constrained FTU for a classifier is coNP-complete.
\end{proposition}

\begin{proof}
To show that a classifier $\kappa$ fails to satisfy constrained FTU, it suffices to 
exhibit $x,y \in \mbb{F}[{\mathcal C}]$ which agree on the unprotected features
but such that $\kappa(x) \neq \kappa(y)$. Thus testing constrained FTU belongs
to coNP. 

To show completeness, let $\Phi(x_1,\ldots,x_n)$ be a DNF 
(a boolean formula in disjunctive normal form) and define 
$\kappa : \{0,1\}^{n+1} \rightarrow \{0,1\}$ by 
$\kappa(x_0,x_1,\ldots,x_n) = \Phi(x_1,\ldots,x_n)$ if $x_0=0$
and $\kappa(x_0,x_1,\ldots,x_n) = 1$ if $x_0=1$. There are no constraints,
so $\mbb{F}[{\mathcal C}]=\mbb{F}$. If we consider
$x_0$ as the only protected feature, then $\kappa$ satisfies constrained FTU
iff $\Phi$ is a tautology. It is well known that testing whether a DNF is
a tautology is coNP-complete~\cite{AroraBarak}, which completes the proof.
\end{proof}

We can state the following corollary of Theorem~\ref{thm:nocstrs},
Proposition~\ref{prop:COND1FTU-exist} and Proposition~\ref{prop:coNPFTU}.

\begin{corollary}
When there are no constraints in ${\mathcal C}$ 
between protected and unprotected features,
testing universal fairness of a classifier is coNP-complete.
%When a classifier is disentangled in $\mathbb{F}[{\mathcal C}]$,
When the set of constraints ${\mathcal C}$ is loose,
testing existential fairness of a classifier is also coNP-complete.
\end{corollary}

Using the notation ${\mathcal C}(x)$ to represent that the instance 
$x \in \mathbb{F}$ satisfies the constraints ${\mathcal C}$, testing 
whether a classifier $\kappa : \mathbb{F}[{\mathcal C}] \rightarrow {\mathcal K}$
satisfies constrained FTU amounts to verifying
\begin{align*} \neg ( \exists x,y \in \mathbb{F}, \exists c_0,c_1 \in {\mathcal K}
\text{ s.t. } {\mathcal C}(x)
\land {\mathcal C}(y) \ \land \\ x_{\mathcal N} = y_{\mathcal N}
\land \kappa(x)=c_0 \land \kappa(y)=c_1 \land c_0 \neq c_1)
\end{align*}
From a practical point of view, we can determine whether a
classifier satisfies constrained FTU by a single call to a
SAT solver provided both ${\mathcal C}(x)$ and 
$\kappa(x)=c_0$ can be coded as boolean formulas. The SAT encoding of
many families of machine-learning models, such as Random Forests~\cite{IzzaRF21}
or Decision Lists~\cite{IgnatievS21}, has been studied in recent work on formal XAI~\cite{Joao22}. On the other hand, checking the existential/universal fairness of a classifier, using the formulas outlined at the start of this section, can be performed with a QBF (Quantified Boolean Formula) solver.

We conclude by studying the complexity of testing looseness and disentangledness.
The former condition on the set of constraints ${\mathcal C}$ guarantees that existential fairness is equivalent to constrained FTU for all classifiers (by Proposition~\ref{prop:COND1FTU-exist}), whereas the latter, involving both $\kappa$ and ${\mathcal C}$, implies existential fairness (by Proposition~\ref{prop:disentangled}).

\begin{proposition}
Testing looseness or disentangledness belongs to $\Pi^{\text{P}}_2$.
\end{proposition}

\begin{proof}
A set of constraints ${\mathcal C}$ is \emph{loose} if for all 
$x \in \mathbb{F}[{\mathcal C}]$, for all $p \in {\mathcal P}$, 
$x_{\{p\}}$ does not strictly subsume $x_{\mathcal N}$ (i.e.
coverage($x_{\{p\}}$) $=$ coverage($x_{\mathcal N}$) or
coverage($x_{\mathcal N}$)$\setminus$coverage($x_{\{p\}}$) $\neq \emptyset$).
This is logically equivalent to the following, which shows that testing looseness belongs to $\Pi^{\text{P}}_2$.
\begin{align*}
\forall x,y \in \mathbb{F}[{\mathcal C}], \ \forall p \in {\mathcal P}, 
\exists z \in \mathbb{F}[{\mathcal C}] \ ( \ \\ (y_{\{p\}} \neq x_{\{p\}} \lor y_{{\mathcal N}}=x_{{\mathcal N}}) \lor (z_{{\mathcal N}}=x_{{\mathcal N}} \land 
z_{\{p\}} \neq x_{\{p\}}) \ )
\end{align*}

A classifier $\kappa$ is
disentangled wrt ${\mathcal C}$ if for all 
$x \in \mathbb{F}[{\mathcal C}]$, ${\mathcal N}$ 
is a weak AXp of %the decision 
$\kappa(x)=c$ that is not strictly 
subsumed by an unfair weak AXp $Q$. This is logically equivalent to
\begin{align*}
\forall x,y \in \mathbb{F}[{\mathcal C}], \ \forall Q \subseteq {\mathcal F}, \ \exists z \in \mathbb{F}[{\mathcal C}] \ ( \ \\
(y_{{\mathcal N}} \neq x_{{\mathcal N}} \lor \kappa(y) = \kappa(x))
\land \\ ((Q \cap P = \emptyset) \ \lor 
(z_Q=x_Q  \land \kappa(z) \neq \kappa(x)) \ \lor \ \\ (y_Q \neq x_Q \lor y_{\mathcal N}=x_{\mathcal N}) \ \lor \ (z_N=x_N \land z_Q \neq x_Q))
\ )
\end{align*}
Hence, testing disentangledness also belongs to $\Pi^{\text{P}}_2$.
\end{proof}

\section{Choice of protected features and causality}  \label{sec:causality}

An orthogonal question to the definition of fairness is the \emph{choice} of the set
of protected features. For example, if gender and race are
protected features,
it seems natural to extend ${\mathcal P}$ to include features
such as `on maternity leave' or `of the same race as $X$', but
including all features constrained by the protected features
is too strong a rule. For example, to be on maternity leave, a person
has to be employed, but being employed is not constrained by gender,
so the feature `in employment' should not be considered as a
protected feature.
To avoid including too many features in ${\mathcal P}$,
one approach is to use a causal graph.
The definition of counterfactual fairness proposed by Kusner et al.~\cite{KusnerLRS17},
is based on a causal graph, unobserved variables and probabilities. 
We adapt their definition to state a purely logical version, rather than a probabilistic version, and refer to it as FTCI (Fairness Through Causal Independence).

Given a set of variables $V$ such that ${\mathcal F} \subseteq V$,
a directed causal graph $G = \langle V,E \rangle$ is a
directed graph in which $(i,j) \in E$ if feature $i$ is
one of the causes of feature $j$.
%\begin{definition} \label{def:ftci}
A classifier $\kappa$ satisfies FTCI (Fairness Through Causal Independence)
with respect to a directed causal graph $G=\langle V,E \rangle$,
where ${\mathcal F} \subseteq V$,
if $\kappa$ can be expressed as a function of the variables in $V$
which are not reachable from ${\mathcal P} \subseteq V$ in $G$.
This definition considers as protected all features reachable in the causal graph from the original set
${\mathcal P}$ of protected features. Observe that, after the extension of the set of protected features, the definition of FTCI is based
on FTU but could be
adapted to take into account \emph{constraints} by applying a different definition
of fairness, such as constrained FTU or existential/universal fairness (Definitions \ref{def:constrainedFTU} and \ref{def:EUfairness}).

According to FTCI, a feature should be considered protected if it is reachable
in the causal graph from a protected feature. 
In causal reasoning, other concepts have emerged, including resolving variables, which are variables influenced by a protected feature in a manner accepted by practitioners as non-descriminatory~\cite{KilbertusRPHJS17}. Hence, the dependence of a classifier’s decisions on a variable that is reachable in the causal graph from a protected feature is problematic, unless via a resolving variable. Our work does not address this matter as determining what is discriminatory or legally acceptable depends on the specific application domain and the judgment of experts. 
This presents a potential 
orthogonal direction for in-depth investigation.

\section{Discussion} \label{sec:disc}
Assessing fairness, either at the level of a particular decision or for a classifier as a whole, depends on the definition of fairness adopted, as different definitions exist to measure it. Whatever the definition of fairness,
the present study shows that ignoring constraints can distort the truth about fairness.
Most potential for distortion occurs when constraints exist between protected
and unprotected features. As pointed out in Section~\ref{sec:causality}, it is not always possible to avoid this
case by protecting all features constrained by sensitive features.

Logical constraints between features can be viewed as a special case of
statistical correlations between features.
We consider that it is important to understand  
fairness in the presence of constraints before
considering the more general problem of correlations. 
Specifically, our study is grounded on the assumption that the set of constraints are known a priori (obtained either from the encoding process or from established rules within the application domain), upon which we proposed definitions to assess the fairness of classifiers.
A generalization of this work can be achieved by first identifying relationships among dataset features and then applying the findings of this study.

As a measure of fairness, we suggest using the existence of a prime-implicant explanation
that does not contain any protected feature. We consider our definition to be of interest, 
as it reveals the underlying \emph{reasons} for each decision made by the classifier. 
However, the increase in computational complexity raises some concerns. Although the definitions discussed in this paper are reasonable and provide a coherent formalisation of the notion of fairness, from a practical point of view, 
their high complexity
makes them difficult to apply. Addressing this computational complexity issue is a potential avenue for future research.
As a first step in this direction, we have identified less computationally expensive
conditions, including looseness of the set of constraints
which imply that the existential fairness
of a classifier can be solved by a single call to a SAT solver.

Two definitions of fairness are proposed in this paper based on PI-explanations, 
namely existential fairness (all decisions have at least one fair explanation) 
and universal fairness (all explanations of all decisions are fair). 
A natural question is whether existential fairness is sufficient in order 
to deem the classifier fair, or whether the more restrictive universal fairness should be imposed. 
We consider that this choice depends on the application domain. In Example~\ref{exAdoptAgain}, 
the existential fairness of the classifier $\kappa'$ may be regarded insufficient 
if not allowing mixed-race couples to adopt is considered to be discrimination,
especially since a universally-fair classifier $\kappa''$ exists.
However, as we saw in Example~\ref{ex0ter}, the constraints may be so strong that
no non-trivial universally-fair classifier exists.

\section{Conclusion} \label{sec:conc}

Fairness is a fundamental requirement for AI systems, especially as they increasingly influence high-risk domains. In this paper, we addressed the fairness of decisions taken by a classifier. A decision can be deemed fair based on the existence of prime-implicant explanations not involving protected features (fair PI-explanations). The study focused on the importance of taking into account constraints between features when assessing fairness. We examined several definitions of fairness of classifiers
in constrained features spaces, based either on unawareness of protected features or
the existence of fair/unfair PI-explanations for each decision.

This work lays novel theoretical foundations for formal fairness
in the presence of constraints between features. We focused on the conceptual and theoretical analysis with practical insights, providing a basis that can be expanded in multiple ways. The present study opens new avenues for future research, some of which have been highlighted in the discussion section. From a practical standpoint, it would be interesting to %conduct more experiments on real-world datasets, 
consider, alongside a priori knowledge of some strict constraints,  correlations between features, both during training and evaluation of models, with the aim of achieving  formal guarantees of fairness.

\section*{Acknowledgments}
This work was supported by the ForML research project ANR-23-CE25-0009.

\bibliography{refsFairness} 

\end{document}